%% file: colt2016-arxiv.tex
\documentclass[10pt]{article}
\usepackage{natbib}
\usepackage[colorlinks = true, linkcolor = blue, citecolor = cyan]{hyperref}
\usepackage{amsmath, amsthm, amsfonts, amssymb}
\usepackage{changepage}

\setcitestyle{authoryear,round,citesep={;},aysep={,},yysep={;}}
\renewcommand{\cite}[1]{\citep{#1}}

\usepackage{graphicx,subfigure,color,tikz}

\usepackage{pdfpages, fullpage}

\usepackage{times}
\usepackage{algorithmic}
\usepackage{algorithm}

{\makeatletter
 \gdef\xxxmark{%
   \protected@write\@auxout{\def\PAGE{ page }}
     {\@percentchar xxx: section \thesubsubsection \PAGE \thepage}%
   \expandafter\ifx\csname @mpargs\endcsname\relax % in minipage?
     \expandafter\ifx\csname @captype\endcsname\relax % in figure/caption?
       \marginpar{xxx}% not in a caption or minipage, can use marginpar
     \else
       xxx % notice trailing space
     \fi
-   \else
     xxx % notice trailing space
   \fi}
 \gdef\xxx{\@ifnextchar[\xxx@lab\xxx@nolab}
 \long\gdef\xxx@lab[#1]#2{{\bf [\xxxmark #2 ---{\sc #1}]}}
 \long\gdef\xxx@nolab#1{{\bf [\xxxmark #1]}}
 \gdef\turnoffxxx{\long\gdef\xxx@lab[##1]##2{}\long\gdef\xxx@nolab##1{}}%
}

\def\<{\langle}
\def\>{\rangle}

\def\E{\mathbb{E}} 
\def\shownotes{1}  %set 1 to show author notes
\ifnum\shownotes=1
\newcommand{\authnote}[2]{{$\ll$\textsf{\footnotesize #1 notes: #2}$\gg$}}
\else
\newcommand{\authnote}[2]{}
\fi

\newcommand{\bx}{\mathbf{x}}
\newcommand{\hide}[1]{}

\newtheorem{thm}{Theorem}
\newtheorem{prop}[thm]{Proposition}
\newtheorem{cor}[thm]{Corollary}
\newtheorem{lem}[thm]{Lemma}

\newtheorem*{defn}{Definition}

\newtheorem*{rem}{Remark}

\numberwithin{equation}{section}

\title{How to calculate partition functions using convex programming hierarchies: provable bounds for variational methods}
\author{Andrej Risteski \thanks{Princeton University, Computer Science Department. Email: risteski@cs.princeton.edu. This work was supported in part by NSF grants CCF-0832797, CCF-1117309, CCF-1302518, DMS-1317308, Sanjeev Arora's Simons Investigator Award, and Simons Collaboration Grant.} }

\begin{document}

\maketitle

\begin{abstract}

We consider the problem of approximating partition functions for Ising models. We make use of recent tools in combinatorial optimization: the Sherali-Adams and Lasserre convex programming hierarchies, in combination with variational methods to get algorithms for calculating partition functions in these families. These techniques give new, non-trivial approximation guarantees for the partition function \emph{beyond the regime of correlation decay}. They also \emph{generalize} some classical results from statistical physics about the Curie-Weiss ferromagnetic Ising model, as well as \emph{provide a partition function counterpart} of classical results about max-cut on dense graphs \cite{arora1995polynomial}. With this, we connect techniques from two apparently disparate research areas -- optimization and counting/partition function approximations. (i.e. \#-P type of problems). 

Furthermore, we design to the best of our knowledge the first \emph{provable, convex} variational methods. Though in the literature there are a host of convex versions of variational methods \cite{wainwright2003tree, wainwright2005new, heskes2006convexity, meshi2009convexifying}, they come with no guarantees (apart from some extremely special cases, like e.g. the graph has a single cycle \cite{weiss2000correctness}). We consider dense and low threshold rank graphs, and interestingly, the reason our approach works on these types of graphs is because local correlations \emph{propagate} to global correlations -- completely the opposite of algorithms based on \emph{correlation decay}. In the process we design novel \emph{entropy approximations} based on the low-order moments of a distribution.

Our proof techniques are very simple and generic, and likely to be applicable to many other settings other than Ising models. \footnote{This paper was accepted for presentation at Conference on Learning Theory (COLT) 2016}

\end{abstract}

\input{introduction-colt-nomatchings}

\input{variational-colt-nomatchings}

\input{hierarchies-colt-nomatchings} 

\input{entropyrespecting-colt}

%\include{general} 

\input{dense-COLT-nomatchings} 

%\include{thresholdrank}

%\input{ferromagnetic} 

%\input{matching-colt}

%\input{matching2-colt} 

\input{conclusion}

\nocite{*}

\bibliographystyle{plainnat}
\bibliography{colt2016-arxiv}

%\appendix 

%\input{matchingappendix-colt} 

\end{document}

%% file: introduction-colt-nomatchings.tex
\section{Introduction} 

Calculating partition functions is a common task in machine learning: for a distribution $p$ over a domain $\mathcal{D}$, specified up to normalization i.e. 
$ p(\mathbf{x}) \propto f(\mathbf{x}), \mathbf{x} \in \mathcal{D}$ for some explicit function $f(\mathbf{x})$, we want to calculate the \emph{partition function} (i.e. the normalization constant) $\sum_{\mathbf{x} \in \mathcal{D}} f(\mathbf{x})$.\footnote{$\mathcal{D}$ can also be continuous of course, in which case the sum becomes an integral, though in this paper we only will be concerned with discrete domains.} This task arises naturally in almost any problem involving learning, performing inference (i.e. calculating marginals) over graphical models, or estimating posterior distributions in latent variable models.  

Broadly, two approaches are used for calculating partition functions: one is based on using Markov Chains to sample from the distribution $p$; the other is variational methods, which involve characterizing the partition function as the solution of a certain (intractable) optimization problem over the polytope of valid distributions over $\mathcal{D}$. In theory, the former are much better studied, the crowning achievements of which are probably \cite{jerrum2004polynomial} and \cite{jerrum1993polynomial}, who proved certain Markov Chains mix rapidly in the case of permanent with non-negative entries and the ferromagnetic Ising model. %As a side product a lot of deep mathematics results have been spurred on, including, but not limited to various results about correlation decay, uniqueness of Gibbs measure, log-Sobolev inequalities, etc. 

In practice however, variational methods are quite popular \cite{wainwright2008graphical, blei2003latent, blei2016variational}. There are various reasons for this, the main being that they can be quite a bit faster than Markov Chain methods \footnote{Markov Chain methods always produce the right answer in the end, but might take longer to converge; variational methods are based on solving an optimization problem, for which it is potentially possible to get stuck in a local optimum, but generally convergence is faster} and they tend to be easier to parallelize. With the exception of belief propagation (which can be viewed as a particular way to solve a certain \emph{non-convex} relaxation of the optimization problem for calculating the partition function \cite{yedidia2003understanding}) there is essentially no theoretical understanding. Additionally, the guarantees for belief propagation usually apply only in the regime of decay of correlations and locally tree-like graphs.  

The contributions of our paper are two-fold. 

First, we bring to bear recent tools in combinatorial optimization: the Sherali-Adams and Lasserre convex programming hierarchies, in combination with variational methods to get algorithms for calculating partition functions of Ising models. These techniques give new, non-trivial approximation guarantees for the partition function \emph{beyond the regime of correlation decay}. They also \emph{generalize} some classical results from statistical physics about the Curie-Weiss ferromagnetic Ising model, as well as \emph{provide a partition function counterpart} of classical results about max-cut on dense graphs \cite{arora1995polynomial}. With this, we connect techniques from two apparently disparate research areas -- optimization and counting/partition function approximations. (i.e. \#-P type of problems). 

Second, we design to the best of our knowledge the first \emph{provable, convex} variational methods. Though in the literature there are a myriad of convex versions of variational methods \cite{wainwright2003tree, wainwright2005new, heskes2006convexity, meshi2009convexifying}, they come with no guarantees at all (except in some extremely special cases, like e.g. the graph has a single cycle \cite{weiss2000correctness}). Our methods tackle dense and low threshold rank graphs, and interestingly, the reason our approach works on these types of graphs is because local correlations \emph{propagate} to global correlations -- which is completely the opposite of algorithms based on \emph{correlation decay}. In the process we design novel \emph{entropy approximations} based on the low-order moments of a distribution. %Apart from the Bethe approximation \cite{yedidia2003understanding}, this is the first entropy approximation which comes with guarantees. 

Our proof methods are extremely simple and generic and we believe they can be applied to many other families of partition functions. 

Finally, one more important reason to study variational methods (albeit more theoretical in nature) is derandomization, since variational methods are usually deterministic. The gap between the state of the art in partition function calculation with and without randomization is huge. For instance, for the case of calculating permanents of non-negative matrices the algorithm due to \cite{jerrum2004polynomial} gets a factor $1+\epsilon$ approximation in time $\mbox{poly}(n,1/\epsilon)$ with high probability (i.e. it's an FPRAS). In contrast, the best deterministic algorithm due to \cite{gurvits} achieves only a factor $2^n$ approximation in time $\mbox{poly}(n)$. (To make the situation even more drastic, the approach in \cite{gurvits} can at best lead to a factor $\sqrt{2}^n$ approximation \cite{avi}.)

\section{Overview of results} 
\label{s:overview}

%We apply our methods to \emph{Ising models} though they are very generic, so it's quite likely they can be applied to many more classes of partition functions, and we leave this for future work.

%Let's consider Ising models first. Our first contribution is to recast and appropriately generalize known results on the \emph{mean-field} Ising model in physics using the language of convex hierarchies. Recall, in physics the case of the ferromagnetic Ising model on a complete graph is usually known as the \emph{mean-field} or \emph{Curie-Weiss} model, and is known for being easily ``solved'' asymptotically -- meaning the value of the partition function
%\footnote{Actually, $\frac{1}{n}\log Z$, where $Z$ is the partition function} 
%is easy to calculate asymptotically. This is usually proved by ``reparametrizing'' the expression for the partition function in terms of the magnetization \footnote{The magnetization is $\sum_{i=1}^n X_i$} and estimating this expressing using Stirling's formula. Unfortunately, this approach is tied to the symmetry of the model, i.e. the potentials between all edges being equal. We recover the above claim using hierarchies of convex programs and a new type of entropy approximation we introduce -- in the process generalizing it quite a bit.
\sloppy 
We focus on \emph{dense} Ising models first: an Ising model $p(\mathbf{x}) \propto \exp\left(\sum_{i,j} J_{i,j} \mathbf{x}_i \mathbf{x}_j\right), \mathbf{x} \in \{-1,1\}^n$ is $\Delta$-\emph{dense} if it satisfies $\Delta |J_{i,j}| \leq \frac{J_T}{n^2}, \forall i,j \in [n]$, where $J_T = \sum_{i,j} |J_{i,j}|$. 

This is a natural generalization of the typical way to define density for combinatorial optimization problems (see e.g. \cite{yoshida}). To see this consider a graph $G = (V,E)$ with $|E| = c n^2$. For optimization problems like max-cut or more generally CSPs, we care about objectives that look like 
$$\mathbb{E}_{e \in E} f(e) = \sum_{e \in E} \frac{1}{|E|} f(e)$$
for some function $f$. Hence, the ``weight'' in front of each pair $(i,j)$ in the objective is 0 if there is no edge or $\frac{1}{|E|}$. This corresponds to $\Delta = \frac{1}{c}$ in our definition. For partition function problems, however, scale matters (i.e. we cannot assume $\sum_{i,j} J_{i,j} = 1$), so the above generalization appears very organic. 

\begin{thm} For $\Delta$-dense Ising models, there is an algorithm based on Sherali-Adams hierarchies which achieves an additive approximation of $\epsilon J_T$ to $\log \mathcal{Z}$, where $\mathcal{Z} = \sum_{\mathbf{x} \in \{-1,1\}^n} \exp\left(\sum_{i,j} J_{i,j} \mathbf{x}_i \mathbf{x}_j\right)$ and runs in time $n^{O\left(\frac{1}{\Delta \epsilon^2}\right)}$. 
\label{t:infdense}
\end{thm}

Our second contribution are analogous claims for Ising models whose potentials look like low rank matrices. (More precisely, adjacency matrices of \emph{low threshold rank} graphs, a concept introduced by \cite{arora2010subexponential} in the context of their algorithm for Unique Games.)

Concretely, an Ising model $p(\mathbf{x}) \propto \exp\left(\sum_{i,j} J_{i,j} \mathbf{x}_i \mathbf{x}_j\right), \mathbf{x} \in \{-1,1\}^n$ is \emph{regular} if $\sum_j |J_{i,j}| = J', \forall i$. The adjacency matrix of a regular Ising model is the matrix $A_{i,j} = |J_{i,j}|/J'$. Then, we show: 

\begin{thm}   There is an algorithm based on Lasserre hierarchies which achieves an additive aproximation of $\epsilon n J'$ to $\log \mathcal{Z}$, where $\mathcal{Z} = \sum_{\mathbf{x} \in \{-1,1\}^n} \exp\left(\sum_{i,j} J_{i,j} \mathbf{x}_i \mathbf{x}_j\right)$, and runs in time $n^{\mbox{rank}(\Omega(\epsilon^2))/\Omega(\epsilon^2)}$, where $\mbox{rank}(\tau)$ is the number of eigenvalues of the adjacency matrix $A$ greater than or equal to $\tau$. 
\label{t:infsparse}
\end{thm}   

It's interesting that this property of the graph, previously introduced for purposes of combinatorial optimization problems like small-set expansion, Unique Games \cite{steurer2010complexity, arora2010subexponential}, also helps with counting type problems. 

Note that since we prove additive factor guarantees to $\log Z$, using the fact the $e^{\epsilon} \leq 1 + 2 \epsilon$ for small enough $\epsilon$, we can easily turn them to multiplicative factor guarantees on $Z$. While these guarantees are not as strong as one usually gets in the correlation decay regime (i.e. $1+\epsilon$ multiplicative factor approximations to $\mathcal{Z}$ in time $\mbox{poly}(n, \frac{1}{\epsilon})$), to the best of our knowledge, these are the first approximations guarantees for $\mathcal{Z}$ when correlation decay does not hold. We discuss interesting regimes of the potentials $J_{i,j}$ in Section~\ref{s:temperatures}.

\subsection{Outline of the techniques} 

Our approach can be summarized as follows. We first express the value of the log-partition function as the solution of a certain (intractable) optimization problem, by using a variational characterization of the log-partition function dating all the way back to Gibbs. (See Lemma \ref{l:klising}.) To be more precise, we express it as $\log \mathcal{Z} = \max_{\mu \in \mathcal{M}} \{E(\mu) + H(\mu)\}$, where $\mathcal{M}$ is the polytope of distributions over $\{-1,1\}^n$, $E(\mu)$ is an \emph{average energy} term, which dependes on pairwise marginals of $\mu$ only, and $H(\mu)$ is the Shannon entropy of $\mu$.  

The source of intractability comes from the fact that we cannot optimize over the polytope $\mathcal{M}$: we will instead optimize over a larger polytope $\mathcal{M'}$, which will come by considering \emph{pseudo-distributions}, derived from either Sherali-Adams or Lasserre hierarchies. Additionally, we need to design a relaxation of $H(\mu)$, since in general we cannot hope to express the entropy of a distribution as a function its low-order marginals only. 

The entropy relaxation $\tilde{H}(\mu)$ needs to satisfy $\tilde{H}(\mu) \geq H(\mu)$ for $\mu \in \mathcal{M}$ and needs to be concave in the variables used in the Sherali-Adams and Lasserre relaxations. The relaxation we use (See Section \ref{d:augmf}) will be based upon the chain rule for entropy, so it will be easy to prove that it upper bounds $H(\mu)$ (Proposition \ref{p:augmfentropy}). 

The analysis of the quality of the relaxation proceeds by \emph{rounding} the pseudo-distributions to an actual distribution. This is slightly different from the roundings in combinatorial optimization, as there we only care about producing a \emph{single} good $\{-1,1\}^n$ solution. Here, because of the entropy term, we must crucially produce a \emph{distribution} over $\{-1,1\}^n$. The observation then is that we can view \emph{correlation rounding}, a rounding previously used in works on combinatorial optimization \cite{brs, yoshida} as producing a distribution over $\{-1,1\}^n$ which has the same entropy as the $\tilde{H}(\mu)$ we defined. (Theorems \ref{t:denseising}, \ref{t:sparseising}). %The analysis here needs to be quite careful, since we care about producing \emph{additive} approximations to $\log Z$. 

%In the case of counting matchings, the analysis is complicated by the fact that we must produce a distribution over matchings, so there is a hard constraint that no vertex can have more than one other vertex matched to it. This is essentially dealt with by Lemma~\ref{l:collisiongap} by considering a particular kind of ``local correlation'' related to the ``collision probability'' of independent rounding, and relating it to the global correlation in correlation rounding.   
%We hope these types of results can be combined with decomposition results in the spirit of \cite{arora2010subexponential, jerrum1996mildly} to get better deterministic algorithm for arbitrary non-negative permanents. 

%% file: variational-colt-nomatchings.tex
\section{Preliminaries}
\label{s:isingalgos}

We proceed with designing approximation algorithms for partition functions of Ising models first. Recall, an Ising model is a distribution $p:\{-1,1\}^n \to [0,1]$ that has the form 
$ p(\mathbf{x}) \propto \exp\left(\sum_{i,j=1}^n J_{i,j} \mathbf{x}_i \mathbf{x}_j\right)$ and its partition function is $\mathcal{Z} = \sum_{\bx \in \{-1,1\}^n} \left(\sum_{i,j=1}^n J_{i,j} \mathbf{x}_i \mathbf{x}_j\right)$.\footnote{There are many generalizations of this, allowing linear or higher order terms, as well as different domains than $\{-1,1\}^n$. Most results we prove can be generalized appropriately to these settings completely mechanically, so for clarity sake we focus on this case.} 

They are very commonly used in practical applications in machine learning because of their flexibility (and other appealing properties like being max-entropy distributions subject to moment constraints), and are extensively studied in theoretical computer science, statistical physics and probability theory. 
 
A full survey is out of the scope of this paper, but we just mention that it can be shown that approximating Z within any polynomial factor is NP-hard for general potentials $J_{i,j}$ \cite{jerrum1993polynomial}. When the potentials $J_{i,j}$ are all non-negative (also known as the \emph{ferromagnetic} Ising model), \cite{jerrum1993polynomial} exhibit an FPRAS for computing $\mathcal{Z}$.

%We first set up the basicools we will be using for both the results on Ising models and those on non-negative permanents. Parts of these tools are a brief review of known results, and parts are new. 

Let us set up the basic tools we will be using. 

\subsection{Variational methods} 
\label{sec:var} 

One of the main ideas all the algorithms will use is the following simple lemma, which characterizes $Z$ as the solution of an optimization problem. It essentially dates back to Gibbs \cite{ellis2012entropy}, who used it in the context of statistical mechanics, though it has been rediscovered by machine learning researchers \cite{wainwright2008graphical,yedidia2003understanding}. For completeness, we reprove it here: 

\begin{lem} [Variational characterization of $\log \mathcal{Z}$]
\label{l:klising} 
 For any distribution $\mu: \{-1,1\}^n \to [0,1]$, 
 $$\sum_{i,j} J_{i,j} \mathbb{E}_{\mu} \left[ \mathbf{x}_i \mathbf{x}_j \right] + H(\mu) \leq \log \mathcal{Z} $$
 with equality at $\mu = p$. 
\end{lem}
\begin{proof}
 For any distribution $\mu:\{-1,1\}^n \to [0,1]$, we can write the KL divergence between $\mu$ and $p$ as 
 $$KL(\mu || p) = \mathbb{E}_{\mu} \left[\log \mu ( \bx)\right] - \mathbb{E}_{\mu} \left[\log p(\bx)\right] = -H(\mu) - \sum_{i, j} J_{i,j} \mathbb{E}_{\mu} \left[ \mathbf{x}_i \mathbf{x}_j \right] + \log \mathcal{Z}$$ 
 
 Since the KL divergence is always non-negative,   
 $\displaystyle -H(\mu) - \sum_{i, j} J_{i,j} \mathbb{E}_{\mu} \left[ \mathbf{x}_i \mathbf{x}_j \right] + \log \mathcal{Z} \geq 0$. 
 Hence,
 $\displaystyle \log \mathcal{Z} \geq H(\mu) + \sum_{i, j} J_{i,j} \mathbb{E}_{\mu} \left[ \mathbf{x}_i \mathbf{x}_j \right]$ 
 which proves the first claim of the lemma. However, equality is achieved whenever the KL divergence is 0, which happens when $\mu = p$. This finishes the second part of the lemma. 
\end{proof} 

An immediate consequence of the above is the following: 
\begin{cor} $\log \mathcal{Z} = \max_{\mu \in \mathcal{M}} \left \{ \sum_{i \sim j} J_{i,j} \mathbb{E}_{\mu} \left[ \mathbf{x}_i \mathbf{x}_j \right] + H(\mu) \right \}$, where $\mathcal{M}$ is the polytope of distributions over $\{-1,1\}^{n}$.
\label{c:cor1} 
\end{cor}
 
We will use the above corollary as follows: instead of considering $\mu \in \mathcal{M}$, which is a polytope we cannot optimize over in polynomial time, we will consider $\mu \in \mathcal{M}'$, for a polytope $\mathcal{M'}$ satisfying $\mathcal{M} \subseteq\mathcal{M'}$, and feasible to optimize over. In fact, $\mathcal{M'}$ will be a polytope of \emph{pseudo-distributions}, associated with either Sherali-Adams or Lasserre hierarchies. This idea is not new -- it has appeared implicitly or explicitly in works on various types of belief propagation. \cite{wainwright2008graphical} 

The novel thing is how we handle the entropy portion of the objective. Since $\mu \in \mathcal{M'}$ is no longer necessarily a distribution, we need to design surrogates for the entropy of $\mu$. A popular choice in the literature is the so-called \emph{Bethe} entropy, which roughly arises by taking the expression for the entropy of $\mu$ in terms of the pairwise marginals when the graph is a tree. (Of course, this expression is exact \emph{only} if the graph is a tree. \cite{yedidia2003understanding}) However, this approximation is not a relaxation of $\log \mathcal{Z}$ in the standard sense -- the Bethe entropy is not an upper bound of the entropy, and the constructed approximation to $\log \mathcal{Z}$ is not concave in general, so the analysis proceeds by analyzing the belief propagation messages directly.\footnote{This approach usually works for graphs that are locally-tree-like (i.e. don't have short cycles), and for which some form of correlation decay holds.} 

We take a completely different approach. To get a proper relaxation for $\log \mathcal{Z}$, we design \emph{functionals} $\tilde{H}(\mu)$ defined on $\mu \in \mathcal{M}'$, s.t. $\tilde{H}(\mu) \geq H(\mu)$ whenever $\mu \in \mathcal{M}$. In brief, we will use the following Corollary to~\ref{c:cor1}: 

\begin{cor} If $\mathcal{M} \subseteq \mathcal{M'}$ and $H(\mu) \leq \tilde{H}(\mu)$ for $\mu \in \mathcal{M}$, then   
$$\log \mathcal{Z} \leq \max_{\mu \in \mathcal{M'}} \left \{\sum_{i \sim j} J_{i,j} \mathbb{E}_{\mu} \left[ \mathbf{x}_i \mathbf{x}_j \right] + \tilde{H}(\mu) \right \}$$ 
\label{c:cor2} 
\end{cor}

Subsequently, we will \emph{round} the pseudo-distributions to actual distributions, in a manner that doesn't lose too much in terms of the value of the objective function.

%% file: hierarchies-colt-nomatchings.tex
\subsection{Sherali-Adams and Lasserre hierarchies} 

We will be strongly using \emph{hierarchies of convex relaxations}, capturing constraints on low-order moments and marginals of distributions. These are where our polytope $\mathcal{M'}$ will come from. While convex hierarchies have recently become relatively well-known in theoretical computer science, we still provide a (very) brief overview for completeness sake. For more details, the reader can consult \cite{brs, barak2014rounding, laurent2009sums}. 

Recall, we are considering relaxations of the polytope of distributions over $\{-1,1\}^n$. 
The $k$-level \emph{Sherali-Adams} hierarchy (henceforth SA($k$)) has variables $\mu_S(\mathbf{x}_S), \mathbf{x}_S \in \{-1,1\}^{|S|}$ specifying local distributions over all subsets $S \subseteq [n]$, $|S| \leq k$. The distributions $\mu_S:\{-1,1\}^{|S|} \to [0,1]$ and $\mu_T:\{-1,1\}^{|T|} \to [0,1]$, for any $S$, $T$ s.t. $|S \cup T| \leq k$ must be ``consistent'' on $S \cap T$. More precisely, it's the case that 
$$\Pr_{\mathbf{x}_S \sim \mu_S}[\mathbf{x}_{S \cap T} = \alpha] = \Pr_{\mathbf{x}_T \sim \mu_T}[\mathbf{x}_{S \cap T} = \alpha], \forall S,T \subseteq [n], |S \cup T| \leq k $$ 
The fact that these constraints can be written as a linear program is well-known. (See e.g. \cite{brs}) 

We can also define a \emph{conditioning} operation thanks to the existence of these local distributions. More precisely, for a vertex $v$, \emph{conditioning} on $v$ involves sampling $v$ according to the local distribution $\mu_{\{v\}}$. This operation specifies a solution to the $k-1$-st level SA hierarchies: just define $\mu_{S}(\mathbf{x}_S) = \mu_{S \cup \{v\}}(\mathbf{x}_{S \cup v})$. 

The additional power we get from the k-th level of the \emph{Lasserre} hierarchy (henceforth LAS($k$)) is that the semidefinite program provides vectors $v_{S,\alpha}$ for each subset $S$ and possible assignment of values $\alpha$ to it, s.t. $\langle v_{S,\alpha}, v_{T,\beta} \rangle  = \Pr_{\mu_{S \cup T}} (\mathbf{x}_S = \alpha, \mathbf{x}_T = \beta)$, if $|S \cup T| \leq k$.  

%The properties of the Sum-of-Squares and Sherali-Adams hierarchies we will need is that for each subset $S$, $|S| \leq k$, they specify consistent local marginals. The Sum-of-Squares hierarchy furthermore 

%% file: entropyrespecting-colt.tex
\section{Entropy respecting roundings} 

In this section we consider the functionals acting as surrogates for entropy. Recall, these need to be upper bounds on the entropy of a distribution $\mu$ on which we have essentially no handle other than having the first few moments. A clear candidate to do this is the \emph{chain rule}.

Notice that for any set $S$ of size at most $k$, where $k$ is the number of levels of the Sherali-Adams or Lasserre hierarchy, $H(\mu_S)$ is a well-defined quantity: it's exactly 
$$H(\mu_S) = \sum_{ \bx_S \in \{-1,1\}^{|S|}} \mu_S(\bx_S) \log(\mu_s(\bx_S))$$ 
Since these local quantities are essentially all the information about the joint distribution $\mu$ we have, our functional must involve such quantities only. 

%An often used approximation in physics is the \emph{mean field} approximation, where one assumes that the distribution $p(\bx)$ is such that it factorizes as $p(\bx) = \Pi_{i=1}^n p_i(\bx_i)$.Inspired by this, we will consider the following pseudo-entropy functional $\tilde{H}(\mu)$: 

The simplest functional one can design surely is the following: 
\begin{defn} The \emph{mean-field pseudo-entropy functional} $H_{\text{MF}}(\mu)$ is defined as $H_{\text{MF}}(\mu) = \sum_{i=1}^n H(\mu_i)$. 
\end{defn} 

\begin{rem} Note, this is \emph{not} the same as the usual mean-field approximation in statistical physics. The mathematical program analogue of that approximation would be to enforce that $\E_{\mu} [\mathbf{x}_i \mathbf{x}_j] = \E_{\mu} [\mathbf{x}_i] \E_{\mu} [\mathbf{x}_j]$ -- which would result in a non-convex relaxation generally. We think the name is appropriate though, since the bound on the entropy is \emph{mean-field}, i.e. results by treating $\mu$ as if it were a product distribution. 
\end{rem} 

Almost trivially for any $\mu \in \mathcal{M}$, the following proposition holds: 
\begin{prop} For any distribution $\mu:\{-1,1\}^n \to [0,1]$, $H(\mu) \leq H_{\text{MF}}(\mu)$ 
\label{p:mfentropy} 
\end{prop} 
\begin{proof}
By the chain rule, $H(\mu) = \sum_{i=1}^n H(\mu_{i}|\mu_{[i-1]})$, where $[i-1]$ denotes the set $\{1,2, \dots, i-1\}$ and $H(X|Y)$ is the conditional entropy of $X$ given $Y$. However, since $H(\mu_{i}|\mu_{[i-1]}) \leq H(\mu_{i})$ the claim trivially holds.  
\end{proof} 

We will also consider generalizations of the above -- where before applying the above "mean-field" bound on the entropy, one can condition on a small subset first. Namely, 
\begin{defn} The \emph{augmented mean-field pseudo-entropy functional} for subsets of size $k$, $H_{\text{aMF},k}(\mu)$ is  defined as $H_{\text{aMF},k}(\mu) = \min_{|S| \leq k} \left \{H(\mu_{S}) + \sum_{i \notin S} H(\mu_i|\mu_S) \right \}$. 
\label{d:augmf}
\end{defn} 

The same proof as in Proposition~\ref{p:mfentropy} implies:  
\begin{prop} $H(\mu) \leq H_{\text{aMF}, k}(\mu)$ 
\label{p:augmfentropy}
\end{prop} 

Furthermore, it's quite easy to show that $H_{\text{aMF},k}(\mu)$, like $H_{\text{MF}}(\mu)$, is a concave function. 

\begin{lem} The pseudo-entropy functional $H_{\text{aMF},k}(\mu) = \min_{|S| \leq k} \left \{H(\mu_S) + \sum_{i \notin S} H(\mu_i|\mu_S) \right \}$ is concave in the variables $\{\mu_{S \cup \{i\}}\left(\bx_{S \cup \{i\}}\right) \vert |S| \leq k, i \in [n]\}$. 
\end{lem} 
\begin{proof} 
Since  $H_{\text{aMF},k}(\mu) = \min_{|S| \leq k} \left \{H(\mu_S) + \sum_{i \notin S} H(\mu_i|\mu_S) \right \}$, and the minimum of concave functions is concave, all we need to show is that $H(\mu_S) + \sum_{i \notin S} H(\mu_i|\mu_S)$ is concave for all $S$. It's well known that entropy is a concave function, so $H(\mu_S)$ is concave. 
What remains to be shown is that $\sum_{i \notin S} H(\mu_i|\mu_S)$ is concave. But, since the sum of concave functions is concave, it suffices to prove $H(\mu_i|\mu_S)$ is concave. 

The proof of this is essentially the same as the proof of concavity of entropy. Abusing notation a bit, we will denote as $\mu_A |\bx_{B}$ the conditional distribution on the variables in $A$, conditioned on the variables in $B$ having the value $\bx_{B}$.   
We recall that 
\begin{align*}
H(\mu_i|\mu_S) &= \sum_{\bx_S \in \{-1,1\}^|S|}\mu_s(\bx_S) H(\mu_i | \bx_s) \\ 
               &= -\sum_{\bx_S \in \{-1,1\}^{|S|}} \sum_{\bx_i \in \{-1,1\}} \mu_s(\bx_S) \mu_{i|\bx_S}(\bx_i) \log (\mu_{i|\bx_S}(\bx_i)) \\
							 &= -\sum_{\bx_S \in \{-1,1\}^{|S|}} \sum_{\bx_i \in \{-1,1\}} \mu_{S \cup \{i\}}(\bx_{S \cup \{i\}}) \log (\mu_{i|\bx_S}(\bx_i)) \\
							 &= -\sum_{\bx_S \in \{-1,1\}^{|S|}} \sum_{\bx_i \in \{-1,1\}} \mu_{S \cup \{i\}}(\bx_{S \cup \{i\}}) \log \left(\frac{\mu_{S \cup \{i\}}(\bx_{S \cup \{i\}})}{\mu_s(\bx_S)}\right) 
\end{align*}		

We rewrite the last expression as a KL divergence as follows:
\begin{equation}  
\label{eq:rewrite}
-\sum_{\bx_S \in \{-1,1\}^{|S|}} \sum_{x_i \in \{-1,1\}} \mu_{S \cup \{i\}}(\bx_{S \cup \{i\}}) \log \left(\frac{\mu_{S \cup \{i\}}(\bx_{S \cup \{i\}})}{\mu_s(\bx_S)\frac{1}{2}}\right)+1= 
 -KL(\mu_{S \cup \{i\}} || (\mu_S \times r)) + 1
 \end{equation}
 where $r$ is a uniform distribution over $\{-1,1\}$. 
 
 Then, if ${\mu}^{\lambda}_{S \cup \{i\}} = \lambda {\mu}^1_{S \cup \{i\}} + (1-\lambda) \mu^2_{S \cup \{i\}}$, we want to show 
$$H(\mu^{\lambda}_i|\mu^{\lambda}_S) \geq \lambda H(\mu^{1}_i|\mu^{1}_S) + (1-\lambda) H(\mu^{2}_i|\mu^{2}_S)$$

 By (\ref{eq:rewrite}) and the convexity of KL divergence, 
 \begin{align*} H(\mu^{\lambda}_i|\mu^{\lambda}_S) & = -KL(\mu^{\lambda}_{S \cup \{i\}} || (\mu^{\lambda}_S \times r)) + 1 \\ 
 &\geq -\lambda KL(\mu^{1}_{S \cup \{i\}} || (\mu^{1}_S \times r)) - (1-\lambda) KL(\mu^{2}_{S \cup \{i\}} || (\mu^{2}_S \times r)) + 1 \\
 &= \lambda H(\mu^{1}_i|\mu^{1}_S) + (1-\lambda) H(\mu^{2}_i|\mu^{2}_S) 
\end{align*}
 which is what we want. 
   
\end{proof} 

%We note that another entropy usually used in practice is the so-called Bethe entropy, which gives rise to the Bethe approximation to the partition function. This approximation however does not give rise to a concave function, and is neither an upper nor a lower bound to the entropy of distribution. Furthermore, the usual justifications for why it makes sense are physics-inspired, and I am fairly confident this is not the way to attack this problem. (In particular, I have little confidence this approximation will work for any ferromagnetic Ising model (dense or sparse), at any temperature.) 

%% file: dense-COLT-nomatchings.tex
\subsection{Dense Ising models} 

We finally turn to designing an algorithm for ``dense'' Ising models. 

There are multiple reasons to study this particular subclass: from the theoretical computer science point of view, we have various PTAS for constraint satisfaction problem when the constraint graph is dense \cite{yoshida, arora1995polynomial} so we might hope to get results better than the worst-case one ones for partition function calculation as well. 

Another motivation comes from \emph{mean-field} ferromagnetic Ising model (also known as the \emph{Curie-Weiss} model \cite{ellis1978statistics}), which is frequently studied as a very simplified model of ferromagnetism because one can get relatively easily results about global properties of the model like the partition function, magnetization, etc. In the mean-field model, each spin interacts (equally strongly) with every other spin. 

We will, in this section, generalize the classical results about the ferromagnetic Curie-Weiss model, as well as provide the natural counterpart of the results in \cite{yoshida, arora1995polynomial} for partition functions. 

Let us first review the standard results about Curie-Weiss. Recall, this model follows the distribution $ p(\mathbf{x}) \propto \exp\left(\sum_{i,j=1}^n \frac{J}{n} \mathbf{x}_i \mathbf{x}_j\right) $, $J >0$. 
It is easy to analyze because $p(\mathbf{x})$ factorizes and can be ``reparametrized'' in terms of the magnetization. Namely, since $\sum_{i,j=1}^n \frac{J}{n} \mathbf{x}_i \mathbf{x}_j = \frac{J}{n} (\sum_i \mathbf{x}_i)^2$, and $(\sum_i \mathbf{x}_i)^2 \in [-n,n]$, one can show \cite{ellis1978statistics}:  

%This is a brief summary of the things I can prove. 
 
\begin{thm} [\cite{ellis1978statistics}] For the Curie-Weiss model, 
$$\log \mathcal{Z} = (1 \pm o(1)) \left( n \max_{m \in [-1,1]} \left( J m^2 + \frac{1-m}{2} \log \frac{1-m}{2} + \frac{1+m}{2} \log \frac{1+m}{2} \right) \right)$$ 
\label{t:ellis}
\end{thm} 

The proof of this theorem involves rewriting the expression for $\mathcal{Z}$ as follows: 
$$ \mathcal{Z} = \sum_{\mathbf{x} \in \{-1,1\}^n} \exp\left(\sum_{i,j} \frac{J}{n} \mathbf{x}_i \mathbf{x}_j\right) = \sum_{l} \exp\left(\frac{J}{n} l^2\right) \cdot n_l$$ 
where $n_l$ is the number of terms where $\sum^n_{i=1} \mathbf{x}_i = l$. Then, using Stirling's formula and some more algebraic manipulation, one can estimate the dominating term in the summation. The claim of the theorem then follows. 

We significantly generalize the above claim using notions from theoretical computer science. The goal is to prove Theorem~\ref{t:infdense}. 

Let $J_T = \sum_{i,j} |J_{i,j}|$. As discussed in Section \ref{s:overview}, we define the following notion of density inspired by the definition of a dense graph in combinatorial optimization \cite{yoshida}:

\begin{defn} An Ising model is $\Delta$-\emph{dense} if $\forall i \neq j, \Delta |J_{i,j}| \leq \frac{J_T}{n^2}$, $\Delta \in (0,1]$. 
\end{defn} 

We will consider the relaxation for $\log \mathcal{Z}$ given by the augmented pseudo-entropy functional and the level $k = O(1/(\Delta \epsilon^2))$ Sherali-Adams relaxation, namely: 

\begin{equation}
\label{eq:conv1}
\max_{\mu \in \mbox{SA}(k), k = O(1/(\Delta \epsilon^2))} \left \{ \sum_{i,j} J_{i,j} \mathbb{E}_{\mu} \left[ \mathbf{x}_i \mathbf{x}_j \right] + H_{\text{aMF}, k}(\mu) \right \}
\end{equation} 

We also recall \emph{correlation rounding} as defined in \cite{brs}. In correlation rounding, we pick a ``seed set'' of a certain size, condition on it, and round the rest of the variables independently. The usual thing to prove is that there is a good ``seet set'' of a small size to condition on. In particular, for the dense case, the following lemma was proven in~\cite{yoshida}: 

\begin{lem}[\cite{yoshida}] There exists a set $S$ of size $k = O(1/(\Delta \epsilon^2))$, s.t. 
$$\left| \sum_{i,j} J_{i,j} \mathbb{E}_{\mu} \left[ \mathbf{x}_i \mathbf{x}_j  | \mathbf{x}_S \right] - \sum_{i,j} J_{i,j} \mathbb{E}_{\mu} \left[ \mathbf{x}_i |\mathbf{x}_S \right] \mathbb{E}_{\mu} \left[ \mathbf{x}_j | \mathbf{x}_S \right] \right| \leq \frac{100}{\Delta k} J_T$$
\label{l:yoshida}
\end{lem} 

With this in hand, we proceed to the main theorem of this section: 

\begin{thm} [Restatement of Theorem \ref{t:infdense}] The output of ~\ref{eq:conv1} is an $\epsilon J_T$ additive approximation to $\log Z$. 
\label{t:denseising}
\end{thm} 
\begin{proof} 
%First, notice that indeed this algorithm is completely deterministic -- all we need to do is solve a convex program. 
The function \ref{eq:conv1} is optimizing is a sum of two terms: $\sum_{i \sim j} J_{i,j} \mathbb{E}_{\mu} \left[ \mathbf{x}_i \mathbf{x}_j \right]$ and an entropy term. Following standard terminology in statistical physics, we will call the former term \emph{average energy}. 

We will analyze the quality of the convex relaxation by exhibiting a \emph{rounding} of the pseudo-distribution to an actual distribution. There is a difference in what this means compared to the roundings we use in combinatorial optimization: there we only care about producing a \emph{single} $\{+1,-1\}$ solution. Here, because of the entropy term, it's essential that we produce a \emph{distribution} over $\{+1,-1\}$ solutions. 
%Notice however that this rounding is only a proof tool! It's never actually performed by the algorithm. Notice further that we'll have to prove the rounding preserves the optimization function to an \emph{additive} factor: since the convex program we are considering approximates $\log Z$. 

We use the fact that correlation rounding can be viewed as producing distributions with a fairly explicit expression for their entropy. Let $S$ be the set of size $O(\frac{1}{\Delta \epsilon^2})$ that Lemma~\ref{l:yoshida} gives. Consider the distribution $\tilde{\mu}(\mathbf{x}) = \mu(\mathbf{x}_S) \Pi_{i \notin S} \mu(\mathbf{x}_i|
\mathbf{x}_S)$\footnote{Notice this is an actual, well-defined distribution, and not only a pseudo-distribution anymore.}. In other words, this is the distribution which rounds the variables in $S$ according to their local distribution, and all other variables independently according to the conditional distribution on $\mathbf{x}_S$. 

Consider the average energy first. By Lemma~\ref{l:yoshida}, 
$$\left| \sum_{i,j} J_{i,j} \mathbb{E}_{\mu} \left[ \mathbf{x}_i \mathbf{x}_j  | \mathbf{x}_S \right] - \sum_{i,j} J_{i,j} \mathbb{E}_{\tilde{\mu}} \left[ \mathbf{x}_i \mathbf{x}_j  | \mathbf{x}_S \right] \right| \leq J_T \epsilon$$

%So, correlation rounding by conditioning on the "best" set $S$ of size $O(\frac{1}{\Delta \epsilon^2})$ produces a distribution $\tilde{\mu}$, s.t. 
%$$\sum_{i,j} J'_{i,j} \mathbb{E}_{\mu} [X_i X_j]-\sum_{i \sim j} J'_{i,j} \mathbb{E}_{\tilde{\mu}} [X_i X_j] \leq \epsilon $$ 

%Hence, $ \sum_{i \sim j} J_{i,j} \mathbb{E}_{\mu} [X_i X_j]-\sum_{i \sim j} J_{i,j} \mathbb{E}_{\tilde{\mu}} [X_i X_j] \leq J \epsilon $. 

Now consider the entropy term. The entropy of the distribution $\tilde{\mu}$ is $H(\tilde{\mu}) = H(\mu_S) + \sum_{i \notin S} H(\mu_i|\mu_S)$. But, since 
 $H_{\text{aMF},k}(\mu) = \min_{|S| \leq k} \left \{H(\mu_{S}) + \sum_{i \notin S} H(\mu_i|\mu_S) \right \}$, $H_{\text{aMF},k}(\mu) \leq H(\tilde{\mu})$ follows. This immediately implies that 

$$ \left( \sum_{i,j} J_{i,j} \mathbb{E}_{\mu} \left[ \mathbf{x}_i \mathbf{x}_j \right] + H_{\text{aMF}, k}(\mu) \right) - \left( \sum_{i,j} J_{i,j} \mathbb{E}_{\tilde{\mu}} \left[ \mathbf{x}_i \mathbf{x}_j \right] + H(\tilde{\mu}) \right) = $$
$$ \left( \sum_{i,j} J_{i,j} \mathbb{E}_{\mu} \left[ \mathbf{x}_i \mathbf{x}_j \right] -  \sum_{i,j} J_{i,j} \mathbb{E}_{\tilde{\mu}} \left[ \mathbf{x}_i \mathbf{x}_j \right] \right) + \left( H_{\text{aMF}, k}(\mu) - H(\tilde{\mu}) \right) \leq J_T \epsilon $$ 

This exactly proves the claim we want. 

\end{proof} 
   
Notice, in the case of the Curie-Weiss model, since $J > 0$, the value of the relaxation \ref{eq:conv1} is at least $J_T$, Theorem \ref{t:denseising} gives a $1+\epsilon$ multiplicative factor approximation to $\log Z$ for any constant $\epsilon$, so generalizes the statement of Theorem \ref{t:ellis} to cases where the potentials $J_{i,j}$ might vary in magnitude and sign. 
   
\subsection{Low threshold rank Ising models}    
\label{s:sparseising}

If we use the added power of the Lasserre hierarchy, we can also handle Ising models whose weights look like low rank matrices. We want to prove Theorem \ref{t:infsparse}. 

We will consider for simplicity in this section \emph{regular} Ising models in the weighted sense, meaning $\sum_{j} |J_{i,j}| = J', \forall i$ \footnote{Though we remind again, all of the claims can be appropriately generalized at the expense of more bothersome notation.}. The \emph{adjacency matrix} of an Ising model will be the doubly-stochastic matrix with entries $|J_{i,j}|/J'$.   
 
Let's recall the definition of threshold rank from \cite{arora2010subexponential}: 

\begin{defn} The $\tau$-threshold rank of a regular graph is the number of eigenvalues of the normalized adjacency matrix greater than or equal to $\tau$. 
\end{defn}

We will, in analogy, define the threshold rank of an Ising model. 

\begin{defn} The $\tau$-threshold rank of a regular Ising model is the number of eigenvalues of its adjacency matrix greater than or equal to $\tau$. 
\end{defn} 

We will consider the following convex program: 
\begin{equation}
\label{eq:conv2}
\max_{\mu \in \mbox{LAS}(k)} \left \{ \sum_{i,j} J_{i,j} \mathbb{E}_{\mu} \left[ \mathbf{x}_i \mathbf{x}_j \right] + H_{\text{aMF}, k}(\mu) \right \}
\end{equation} 

Consider the vectors $v_i, i \in [n]$, s.t. $\langle v_j, v_j \rangle = \E_{\mu} \left[ \mathbf{x}_i \mathbf{x}_j \right]$.  
Then, \cite{brs} prove that when the graph has low threshold rank, ``local'' correlations propagate to ``global'' correlations, and as a consequence of this, there is a set of size at most $\mbox{rank}(\Omega(\epsilon^2))/\Omega(\epsilon^2)$, such that conditioning on it causes the  $\left| \sum_{i,j} J_{i,j} \mathbb{E}_{\mu} \left[ \mathbf{x}_i \mathbf{x}_j  | \mathbf{x}_S \right] - \sum_{i,j} J_{i,j} \mathbb{E}_{\tilde{\mu}} \left[ \mathbf{x}_i \mathbf{x}_j  | \mathbf{x}_S \right] \right|$ to drop below $\epsilon J_T$. More precisely: 

\begin{lem}[\cite{brs}] There exists a set $S$ of size $t \leq \mbox{rank}(\Omega(\epsilon^2))/\Omega(\epsilon^2)$, where $\mbox{rank}(\tau)$ is the $\tau$-threshold rank of the Ising model,  s.t. \footnote{Note, $J_T = n J'$ in this case. } 
$$\left| \sum_{i,j} J_{i,j} \mathbb{E}_{\mu} \left[ \mathbf{x}_i \mathbf{x}_j  | \mathbf{x}_S \right] - \sum_{i,j} J_{i,j} \mathbb{E}_{\mu} \left[ \mathbf{x}_i |\mathbf{x}_S \right] \mathbb{E}_{\mu} \left[ \mathbf{x}_j | \mathbf{x}_S \right] \right| \leq \epsilon J_T$$ 
\end{lem} 

%\begin{lem}[Local-to-Global, \cite{brs}] If $\E_{(i,i') \in G'} \langle v_i, v_{i'} \rangle \geq \epsilon$, the global correlation of the vectors $\E_{(i,i')} |\langle v_i, v_{i'} \rangle| \geq \Omega(\epsilon)/\mbox{rank}_{\geq \Omega(\epsilon)}(G')$
%\end{lem}

Hence, analogously as in Theorem~\ref{t:denseising}, we get: 

\begin{thm} [Restatement of Theorem \ref{t:infsparse}] The output of ~\ref{eq:conv2} is a $\epsilon J_T$ additive approximation to $\log \mathcal{Z}$. 
\label{t:sparseising}
\end{thm} 

\section{Discussion on interpreting the results} 
\label{s:temperatures} 

Since the above results are stated in terms of the additive approximation they provide for $\log \mathcal{Z}$, we discuss how one should interpret them in different ``temperature regimes'' i.e. different scales of the potentials $J_{i,j}$. Note that partition function approximation problems are not scale-invariant, and their hardness is sensitive to the size of the coefficients $J_{i,j}$. 

For simplicity of the discussion, let's focus on the case where there is an underlying graph $G = (V,E)$, such that $J_{i,j} = \pm J$, for $(i,j) \in E(G)$, and 0 otherwise. Furthermore, let's assume the graph $G$ is $d$-regular. 

There are generically three regimes for the problem: 
\begin{itemize} 
\item ``High temperature regime'', i.e. when $|J|  = O\left(\frac{1}{d}\right)$ for a sufficiently small constant in the $O\left(\cdot\right)$ notation. In this case, standard techniques like Dobrushin's uniqueness criterion show that there is correlation decay. This is the regime where generically Markov Chain methods work. Note that using such methods, generally one can get a $(1 + \epsilon)$-factor approximation for $\mathcal{Z}$ in time $\mbox{poly}\left(n,\frac{1}{\epsilon}\right)$, which is unfortunately much stronger than what our method gets in that regime. It would be extremely interesting to see if the methods in our paper can be modified to subsume this regime as well. 
\item ``Around the transition threshold'', i.e. when $|J| = \Theta(\frac{1}{d})$ for a sufficiently large constant in the $\Theta$ notation, such that there is no correlation decay. Generally, unless there is some special structure, Markov Chain methods will provide \emph{no non-trivial} guarantee in this regime -- however, we get an order $\epsilon n$ additive factor approximation to $\log \mathcal{Z}$, which translates to a $(1+\epsilon)^{n}$ factor approximation of $\mathcal{Z}$. We do not, to the best of our knowledge, know how to get such results using \emph{any other methods.}   
\item ``Low temperature regime'', i.e. when $|J| = \omega(1/d)$. In this case, in light of the variational characterization of $\log \mathcal{Z}$ and the fact that the entropy is upper bounded by $n$, the dominating term will typically be the energy term 
$ \sum_{(i,j) \in E(G)} J_{i,j}\mathbb{E}_{\mu} [\mathbf{x}_i \mathbf{x}_j] $, so essentially the quality of approximation will be dictated by the hardness of the optimization problem corresponding to the energy term. (e.g., for the anti-ferromagnetic case, where all the potentials $J_{i,j}$ are negative, the optimization problem corresponding to the energy term is just max-cut, and we cannot hope for more than a constant factor approximation to $\log \mathcal{Z}$ for general (negative) potentials.)  
\end{itemize} 

%Finally, a few remarks. If we care about $1+\epsilon$ factor approximations to the partition function, which is what randomized methods typically achieve, the theorems in this section would give such results in the \emph{high-temperature} regime: in Theorem~\ref{t:denseising}, the potentials $J_{i,j}$ should scale like $O(1/n^2)$, and in Theorem \ref{t:sparseising} like $O(1/n)$. This is reasonable, since in the high-temperature regime, the Ising model specifies a distributions closer to a uniform one. 

%Note in this section we assumed nothing about the sign of the $J_{i,j}$ variables. The results work for anti-ferromagnetic models just as they do for ferromagnetic models. 
%This can be seen both as a weakness and as a strength of the approach: on the one hand, we get more general results; on the other, this means that to derandomize \cite{jerrum1993polynomial}'s algorithm in the ferromagnetic case, something additional will be needed. 

%% file: conclusion.tex
\section{Conclusion} 

We presented simple new algorithms for calculating partition functions in Ising models based on variational methods and convex programming hierarchies. To the best of our knowledge, these techniques give new, non-trivial approximation guarantees for the partition function when correlation decay does not hold, and are the first provable, convex variational methods. Our guarantees are for dense or low threshold rank graphs, and in the process we design novel \emph{entropy approximations} based on the low-order moments of a distribution.

We barely scratched the surface, and we leave many interesting directions open. Our methods are very generic, and are probably applicable to many other classes of partition functions apart from Ising models. One natural candidate is weighted matchings due to the connections to calculating non-negative permanents. 

Another intriguing question is to determine if there is a similar approach that can subsume prior results on partition function calculation in the regime of correlation decay, as our guarantees are much weaker there. This would give a convex relaxation interpretation of these types of results. 
%a new way of designing deterministic algorithms for counting-type problems, based on convex programs and variational methods, and applied it to calculating partition functions of Ising models and counting weighted matchings. They work well in regimes where the underlying graph in the problem is dense or expander-like. The approach is very generic and we hope it will find applications in other domains. 